\documentclass[letterpaper]{article}
\usepackage{hyperref,natbib,subfigure}
\usepackage[accepted]{icml2012}
\usepackage{amsthm,amsfonts,amsmath,amssymb,color,float,graphicx,verbatim}
\usepackage{algorithm,algorithmic}

\newtheorem{theorem}{Theorem}

\newtheorem{lemma}{Lemma}

\newcommand{\E}{\mathbb{E}}

\newcommand{\br}{\mathbf{r}}

\newcommand{\bx}{\mathbf{x}}

\newcommand{\bp}{\mathbf{p}}

\newcommand{\bq}{\mathbf{q}}

\newcommand{\Ocal}{\mathcal{O}}

\newcommand{\Bcal}{\mathcal{B}}

\newcommand{\Pcal}{\mathcal{P}}

\newcommand{\norm}[1]{\|#1\|}

\newcommand{\secref}[1]{Sec.~\ref{#1}}

\renewcommand{\eqref}[1]{Eq.~(\ref{#1})}
\newcommand{\lemref}[1]{Lemma~\ref{#1}}

\newcommand{\thmref}[1]{Thm.~\ref{#1}}

\newcommand{\algref}[1]{Algorithm~\ref{#1}}

\icmltitlerunning{Decoupling Exploration and Exploitation in Multi-Armed Bandits}

\begin{document}

\twocolumn[
\icmltitle{Decoupling Exploration and Exploitation in Multi-Armed Bandits}

\icmlauthor{Orly Avner}{orlyka@tx.technion.ac.il}
\icmlauthor{Shie Mannor}{shie@ee.technion.ac.il}
\icmladdress{Department of Electrical Engineering, Technion}
\icmlauthor{Ohad Shamir}{ohadsh@microsoft.com}
\icmladdress{Microsoft Research New England}


\vskip 0.3in
]

\begin{abstract}
We consider a multi-armed bandit problem where the decision maker can explore and exploit different arms at every round. The exploited arm adds to the decision maker's cumulative reward (without necessarily observing the reward) while the explored arm reveals its value. We devise algorithms for this setup and show that the dependence on the number of arms, $k$, can be much better than the standard $\sqrt{k}$ dependence, depending on the behavior of the arms' reward sequences. For the important case of piecewise stationary stochastic bandits, we show a significant improvement over existing algorithms. Our algorithms are based on a non-uniform sampling policy, which we show is essential to the success of {\em any} algorithm in the adversarial setup. Finally, we show some simulation results on an ultra-wide band channel selection inspired setting indicating the applicability of our algorithms.
\end{abstract}


\section{Introduction}

Multi-armed bandits have long been a canonical framework for studying online learning under partial information constraints. In this framework, a learner has to repeatedly obtain rewards by choosing from a fixed set of $k$ actions (arms), and gets to see only the reward of the chosen action. The goal of the learner is to minimize regret, namely the difference between her own cumulative reward and the cumulative reward of the best single action in hindsight. We focus here on algorithms suited for adversarial settings, which have reasonable regret even without any stochastic assumptions on the reward generating process.

A central theme in multi-armed bandits is the \emph{exploration-exploitation tradeoff}: The learner must choose highly-rewarding actions most of the time in order to minimize regret, but also needs to do some exploration in order to determine which actions to choose. Ultimately, the tradeoff comes from the assumption that the learner is constrained to observe only the reward of the action she picked.

While being a compelling and widely applicable framework, there exist several realistic bandit-like settings, which do not correspond to this fundamental assumption. For example, in ultra-wide band (UWB) communications, the decision maker, also called the ``secondary," has to decide in which channel to transmit and in what way. There are typically many possible channels (i.e., frequency bands) and several transmission methods (power, code used, modulation, etc.; see \cite{UWBBook}). In some UWB devices, the secondary can sense a different channel (or channels) than the one it currently uses for transmission. In fact, in some settings, the secondary cannot sense the channel it is currently transmitting in because of interference. The UWB environment is extremely noisy since it potentially contains many other sources, called ``primaries.'' Some of these sources are sources whose behavior (which channel they use, for how long, and in which power level) can be very hard to predict as they represent a mobile device using WiMAX, WiFI or some other communication protocol. It is therefore sensible to model the behavior of primaries as an adversarial process or a piecewise stationary process. We should mention that UWB networks are highly complex,  with many issues such as power constraints and multi-agency that have been considered in the multi-armed bandit framework \cite{LiuZhao10,AvnerMannor11,LaiJiangPoor08}, but the decoupling of sensing and transmission has not been considered to the best of our knowledge. More abstractly, our work relates to any bandit-like setting, where we are free to query the environment for some additional partial information, irrespective of our actual actions.

In such settings, the assumption that the learner can only observe the reward of the action she picked is an unnecessary constraint, and one might hope that removing this constraint and constructing suitable algorithms would allow better performance. We emphasize that this is far from obvious: In this paper, we will mostly focus on the case where the learner may query just a single action, so in some sense the learner gets the same ``amount of information'' per round as the standard bandit setting (i.e., the reward of a single action out of $k$ actions overall). The goal of this paper is to devise algorithms for this setting, and analyze theoretically and empirically whether the hope for improved performance is indeed justified. We emphasize that our results and techniques naturally generalize to cases where more than one action can be queried, and cases where the reward of the selected action is always revealed (see \secref{sec:discussion}).

Specifically, our contributions are the following:
\begin{itemize}
    \item We present a ``decoupled'' multi-armed bandit algorithm, which is suited to our setting. The algorithm is based on a certain querying distribution, which is adaptive and depends on the distribution by which the actions are actually picked. We show a ``data-dependent'' regret guarantee for the algorithm, which is never worse than that of standard bandit algorithms, and can be much better (in terms of dependence on the number of actions $k$), depending on how the actions' rewards behave.
    \item We prove that in certain settings (in particular, piecewise stochastic rewards), the decoupling assumption allows us to devise algorithms with significantly better performance than \emph{any} possible standard bandit algorithm.
    \item Our algorithms are based on a certain adaptive querying distribution, in contrast to previous works in the stochastic case where the querying distribution was uniform. We show that in some sense, such an adaptive policy is {\em necessary} in an adversarial setting, in order to get performance improvements  compared to standard bandit algorithms.
    \item We perform a preliminary experimental study, corroborating our theoretical findings and indicating that our algorithmic approach indeed leads to improved results, compared to standard approaches.
\end{itemize}

The proofs of our theorems are provided in the appendix of the full version \cite{AMSFullVersion12}.

\textbf{Related Work.} The idea of decoupling exploration and exploitation has appeared in a few previous works, but in different settings and contexts. For example, \cite{YuMa09} discuss a setting where the learner is allowed to query an additional action in a multi-armed bandit setting, but the focus there was on algorithms for stochastic bandits, as opposed to adversarial bandits as we do here. \cite{AgarwalDekelXiao10} study a bandit setting with (one or more) queries per round. However, they focus on the problem of bandit convex optimization, which is much more general than ours, and exploration and exploitation remains coupled in their framework. A different line of work (\cite{Even-DarMannorMansor06,audbumu10,bumust11}) considers multi-armed bandits in a stochastic setting, where the goal is to identify the best action by performing pure exploration. While this work also conceptually ``decouples'' exploration and exploitation, the goal and setting are quite different than ours.

\section{Problem Setting}

We use $[k]$ as shorthand for $\{1,\ldots,k\}$. Bold-face letters represent vectors, and $\mathbf{1}_{A}$ represents the indicator function for an event $A$. We use the standard big-Oh notation $\Ocal(\cdot)$ to hide constants, and $\tilde{\Ocal}(\cdot)$ to hide constants and logarithmic factors. For a distribution vector $\bp$ on the $k$-simplex, we use the notation
\[
\norm{\bp}_{1/2} = \left(\sum_{j=1}^{k}\sqrt{p_j}\right)^2
\]
to describe the `$\ell_{1/2}$'-norm of the distribution. It is straightforward to show that for a distribution vector, this quantity is always in $[1,k]$. In particular, it is $k$ for the uniform distribution, and gets smaller the more non-uniform the distribution is, attaining the value of $1$ when $\bp$ is a unit vector.

Our setting is a variant of the standard adversarial multi-armed bandit framework, focusing (for simplicity) on an oblivious adversary and a fixed horizon. In this setting, we have a fixed set of $k>1$ actions and a fixed known number of rounds $T$. Each action $i$ at each round $t$ has an unknown associated reward $g_i(t)\in[0,1]$. At each round, a learner chooses one of the actions $i_t$, and obtains the associated reward $g_{i_t}(t)$. The basic goal in this setting is to minimize the \emph{regret} with respect to the best single action in hindsight, namely
\[
\max_{i}\sum_{t=1}^{T}g_i(t)-\sum_{t=1}^{T}g_{i_t}(t).
\]
Unless specified otherwise, we make no assumptions on how the rewards $g_i(t)$ are generated (other than boundedness), and they might even be generated adversarially by an agent with full knowledge of our algorithm. However, we assume that the rewards are fixed in advance and do not depend on the learner's (possibly random) choices in previous rounds.

In standard multi-armed bandits, at the end of each round, the learner only gets to know the reward $g_{i_t}(t)$ of the action $i_t$ which was actually picked, but not the reward of other actions. Instead, in this paper we focus on a different setting, where the learner, after choosing an action $i_t$, may \emph{query} a single action $j_t$ and get to see its associated reward $g_{j_t}(t)$. This setting is a (slight) relaxation of the standard bandit setting, since we can always query $j_t=i_t$. However, here it is possible to query an action different than $i_t$. We emphasize that the regret is still measured with respect to the chosen actions $i_t$, and the querying only has informational value. In order to compare our results with those obtainable in the standard setting, we will use the term \emph{standard bandit algorithm} to refer to algorithms which are not free to query rewards, and are limited to receiving the reward of the chosen action. A typical example is the EXP3.P \cite{AuerCesFrSc02}, with a $\tilde{\Ocal}(\sqrt{kT})$ regret upper bound, holding with high probability, or the Implicitly Normalized Forecaster of \cite{AudBub09} with $\Ocal(\sqrt{kT})$ regret.

An interesting variant of our setting is when the learner gets to query more than one action, or gets to see $g_{i_t}(t)$ on top of $g_{j_t}(t)$. Such variants are further discussed in \secref{sec:discussion}.

\section{Basic Algorithm and Results}

In analyzing our ``decoupled'' setting, perhaps the first question one might ask is whether one can \emph{always} get improved regret performance, compared to the standard bandit setting. Namely, that for any reward assignment, the attainable regret will always be significantly smaller. Unfortunately, this is not the case: It can be shown that there exists an adversarial strategy such that the regret of standard bandit algorithms is $\tilde{\Theta}(\sqrt{kT})$, whereas the regret of any ``decoupled'' algorithm will be\footnote{One simply needs to consider the strategy used to obtain the
 $\Omega(\sqrt{kT})$ regret lower bound in the standard bandit setting \cite{AuerCesFrSc02}. The lower bound proof can be shown to apply to a ``decoupled'' algorithm as well. Intuitively, this is because the hardness for the learner stems from distinguishing slightly different distributions based on at most $T$ samples, which has nothing to do with the coupling constraint.} $\Omega(\sqrt{kT})$. Therefore, one cannot hope to always obtain better performance. However, as we will soon show, this can be obtained under certain realistic conditions on the actions' rewards.

We now turn to present our first algorithm (\algref{alg:bandits} below) and the associated regret analysis. The algorithm is rather similar in structure to standard bandit algorithms, picking actions at random in each round $t$ according to a weighted distribution $\bp(t)$ which is updated multiplicatively. The main difference is in determining how to query the reward. Here, the queried action is picked at random, according to a query distribution $\bq(t)$ which is based on but not identical to $\bp(t)$. More particularly, the queried action $j_t$ is chosen with probability
\begin{equation}\label{eq:qdist}
q_{j_t}(t)=\frac{\sqrt{p_{j_t}(t)}}{\sum_{j=1}^{k}\sqrt{p_{j}(t)}}.
\end{equation}
Roughly speaking, this distribution can be seen as a ``geometric average'' between $\bp(t)$ and a uniform distribution over the $k$ actions. See \algref{alg:bandits} for the precise pseudocode.

\begin{algorithm}
\caption{Decoupled MAB Algorithm}
\label{alg:bandits}
\begin{algorithmic}
\STATE \textbf{Input:} Step size parameter $\mu\in [1,k]$, confidence parameter $\delta\in (0,1)$
\STATE Let $\eta=1/\sqrt{\mu T}$, $\beta= 2\eta\sqrt{6\log(3k/\delta)}$ and $\gamma = \eta^2(1+\beta)^2 k^2$
\STATE $\forall~j\in [k]$ let $w_j(1)=1$.
\FOR{$t=1,\ldots,T$}
    \STATE $\forall~j\in [k]$, let $p_j(t)=(1-\gamma)\frac{w_j(t)}{\sum_{l=1}^{k}w_l(t)}+\frac{\gamma}{k}$
    \STATE Choose action $i_t$ with probability $p_{i_t}(t)$
    \STATE Query reward $g_{j_t}(t)$ with probability \\
    $~~~~~~~~~~~
    q_{j_t}(t)=\frac{\sqrt{p_{j_t}(t)}}{\sum_{j}\sqrt{p_{j}(t)}}
    $
    \STATE $\forall~j\in[k]$, let $\tilde{g}_{j}(t)=\frac{1}{q_{j}(t)}\left(g_{j}(t)\mathbf{1}_{j_t=j}+\beta\right)$
    \STATE $\forall~j\in [k]$, let $w_j(t+1)= w_j(t)\exp(\eta \tilde{g}_j(t))$
\ENDFOR
\end{algorithmic}
\end{algorithm}

Readers familiar with bandit algorithms might notice the existence of the common ``exploration component'' $\gamma/k$ in the definition of $p_j(t)$. In standard bandit algorithm, this is used to force the algorithm to explore all arms to some extent. In our setting, exploration is performed via the separate query distribution $q_j(t)$, and in fact, this $\gamma/k$ term can be inserted into the $q_j(t)$ definition instead. While this would be more aesthetically pleasing , it also seems to make our proofs and results more complicated, without substantially improving performance. Therefore, we will stick with this formulation.

Before discussing the formal theoretical results, we would like to briefly explain the intuition behind this querying distribution. Most bandit algorithms (including ours) build upon a standard multiplicative updates approach, which updates the distribution $\mathbf{p}(t)$ multiplicatively based on each action's rewards. In the bandit setting, we only get partial information on the rewards, and therefore resort to multiplicative updates based on an unbiased estimate of them. The key quantity which controls the regret is the variance of these estimates, in expectation over the action distribution $\mathbf{p}(t)$. In our case, this quantity turns out to be on the order of $\sum_{j=1}^{k}p_j(t)/q_j(t)$. Now, standard bandit algorithms, which may not query at will, are essentially constrained to have $q_j(t)=p_j(t)$, leading to an expected variance of $k$ and hence the $k$ in their $\tilde{\Ocal}(\sqrt{kT})$ regret bound. However, in our case, we are free to pick the querying distribution $\bq(t)$ as we wish. It is not hard to verify that $\sum_{j=1}^{k}p_j(t)/q_j(t)$ is minimized by choosing $\bq(t)$ as in \eqref{eq:qdist}, with the value of $\norm{\bp(t)}_{1/2}$. Thus, roughly speaking, instead of dependence on $k$, we get a dependence on $\frac{1}{T}\sum_{t=1}^{T}\norm{\bp(t)}_{1/2}$, as will be seen shortly.

The theoretical analysis of our algorithm relies on the following technical quantity: For any algorithm parameter choices $\mu,\delta$, and for any $v\in [1,k]$, define
\[
P(v,\delta,\mu) =  \Pr\left(\frac{1}{T}\sum_{t=1}^{T}\norm{\bp(t)}_{1/2}>v\right),
\]
where the probability is over the algorithm's randomness, run with parameters $\mu,\delta$, with respect to the (fixed) reward sequence. The formal result we obtain is the following:
\begin{theorem}\label{thm:mab}
Suppose that $T$ is sufficiently large (and thus $\eta$ and $\beta$ sufficiently small) so that $(1+\beta)^2\leq 2$. Then for any $v\in [1,k]$, it holds that with probability at least $1-\delta-P(v,\delta,\mu)$ that the sequence of rewards $g_{i_1}(1),\ldots,g_{i_T}(T)$ returned by Algorithm \ref{alg:bandits} satisfies
\begin{align*}
&\max_{i}\sum_{t=1}^{T}g_i(t)-\sum_{t=1}^{T}g_{i_t}(t)\\
&\;\;
\leq~\tilde{\Ocal}\left(\sqrt{\left(\frac{v^2}{\mu}+\mu+v\right)T}+\frac{k^2}{\mu}
+\frac{k^2}{T^{3/2}}\right)
\end{align*}
where the $\tilde{\Ocal}$ notation hides numerical constants and factors logarithmic in $k$ and $\delta$.
\end{theorem}

At this point, the nature of this result might seem a bit cryptic. We will soon provide more concrete examples, but would like to give a brief general intuition. First of all, if we pick $\mu=v=k$, then $P(v,\delta,\mu)=0$ always (as $\norm{\bp(t)}_{1/2}\leq k$), and the bound becomes $\tilde{\Ocal}(\sqrt{kT})$, holding with probability $1-\delta$, similar to standard multi-armed bandit guarantees. This shows that our algorithm's regret guarantee is \emph{never} worse than that of standard bandit algorithms. However, the theorem also implies that under certain conditions, the resulting bound may be significantly better. For example, if we run the algorithm with $\mu=1$ and have $v=\Ocal(1)$, then the bound becomes $\tilde{\Ocal}\left(\sqrt{T}\right)$ for sufficiently large $T$. This bound is meaningful only if $P(\Ocal(1),\delta,1)$ is reasonably small. This would happen if the distribution vectors $\bp(t)$ chosen by the algorithm tend to be highly non-uniform, since it leads to a small value for $\frac{1}{T}\sum_{t=1}^{T}\norm{\bp(t)}_{1/2}$.



We now turn to provide a concrete scenario, where the bound we obtain is better than those obtained by standard bandit algorithms. Informally, the scenario we discuss assumes that although there are $k$ actions, where $k$ is possibly large, only a small number of them are actually ``relevant'' and have a performance close to that of the best action in hindsight. Intuitively, such cases would lead to the distribution vectors $\bp(t)$ to be non-uniform, which is favorable to our analysis.

\begin{theorem}\label{thm:nonuniformupperbound}
Suppose that the reward of each action is chosen i.i.d. from a distribution supported on $[0,1]$. Furthermore, suppose that there exist a subset $G\subset [k]$ of actions and a parameter $\Delta>0$ (where $|G|,\Delta$ are considered constants independent of $k,T$), such that  the expected reward of any action in $G$ is larger than the expected reward of any action in $[k]\setminus G$ by at least $\Delta$.
Then if we run our algorithm with
\[
\mu = k^{\min\left\{1,~\max\left\{0,~\frac{4}{3}-\frac{1}{3}\log_k(T)
\right\}\right\}},
\]
it holds with probability at least $1-\delta$ that the regret of the algorithm is at most
\[
\tilde{\Ocal}\left(\sqrt{k^{\max\{0,\frac{4}3-\frac{1}{3}\log_k(T)\}}T}
\right),
\]
where the $\tilde{\Ocal}$ notation hides numerical constants and factors logarithmic in $\delta,k$.
\end{theorem}

The bound we obtain interpolates between the usual $\tilde{\Ocal}(\sqrt{kT})$ bound obtained using a standard bandit algorithm, and a considerably better $\tilde{\Ocal}(\sqrt{T})$, as $T$ gets larger compared with $k$. We note that a mathematically equivalent form of the bound is
\[
\max\left\{\left(\frac{k}{T}\right)^{2/3},\left(\frac{1}{T}\right)^{1/2}\right\}T.
\]
Namely, the average per-round regret scales down as $(k/T)^{2/3}$, until $T$ is sufficiently large and we switch to a $(1/T)^{1/2}$ regime. In contrast, the bound for standard bandit algorithms is always of the form $(k/T)^{1/2}$, and the rate of regret decay is significantly slower.

We emphasize that although the setting discussed above is a stochastic one (where the rewards are chosen i.i.d.), our algorithm can cope simultaneously with arbitrary rewards, unlike algorithms designed specifically for  stochastic i.i.d. rewards (which do admit better dependence in $T$, although not necessarily in $k$).

Finally, we note in practice, the optimal choice of $\mu$ depends on the (unknown) rewards, and hence cannot be determined by the learner in advance. However, this can be resolved algorithmically by a standard doubling trick (cf. \cite{CesaBianchiLu06}), without materially affecting the regret guarantee. Roughly speaking, we can guess an upper bound $v$ on $\frac{1}{T}\sum_{t=1}^{T}\norm{\bp(t)}_{1/2}$ and pick $\mu=v$, and if the cumulative sum $\sum \norm{\bp(t)}_{1/2}$ eventually exceeds $Tv$ at some round, then we double $v$ and $\mu$ and restart the algorithm.

\section{Decoupling Provably Helps in some Adversarial Settings}

So far, we have seen how the bounds obtained for our approach are better than the ones known for standard bandit algorithms. However, this doesn't imply that our approach would indeed yield better performance in practice: it might be possible, for instance, that for the setting described in \thmref{thm:nonuniformupperbound}, one can provide a tighter analysis of standard bandit algorithms, and recover a similar result. In this section, we show that there are cases where decoupling provably helps, and our approach can provide performance provably better than any standard bandit algorithm, for information-theoretic reasons. We note that the idea of decoupling has been shown to be helpful in cases reminiscent of the one we will be discussing \cite{YuMa09}, but here we study it in the more general and challenging adversarial setting.

Instead of the plain-vanilla multi-armed bandit setting, we will discuss here a slightly more general setting, where our goal is not to achieve regret with respect to the best single action, but rather to the best sequence of $S>1$ actions. More specifically, we wish to obtain a regret bound of the form
\[
\max_{\stackrel{1=T_{1}\leq T_{2}\leq\ldots\leq T_{S+1}=T}{i^1,\ldots,i^S\in [k]}}
\sum_{s=1}^{S}\sum_{t=T_{s}+1}^{T_{s+1}}g_{i^s}(t)-\sum_{t=1}^{T}g_{i_t}(t).
\]
This setting is well-known in the online learning literature, and has been considered for instance in \cite{HerbsterWarmuth98} for full-information online learning (under the name of ``tracking the best expert'') and in \cite{AuerCesFrSc02} for the bandit setting (under the name of ``regret against arbitrary strategies'').

This setting is particularly suitable when the best action changes with time.
Intuitively, our decoupling approach helps here, since we can exploit much more aggressively while still performing reasonable exploration, which is important for detecting such changes.

The algorithm we use follows the lead of \cite{AuerCesFrSc02} and is presented as \algref{alg:banditsswitch}. The only difference compared to \algref{alg:bandits} is that the $w_j(t+1)$ parameters are computed differently. This change facilitates more aggressive exploration.

\begin{algorithm}
\caption{Decoupled MAB Algorithm For Switching}
\label{alg:banditsswitch}
\begin{algorithmic}
\STATE \textbf{Input:} Step size parameter $\mu\in [1,k]$, confidence parameter $\delta\in (0,1)$, number of switches $S$
\STATE Let $\eta=\sqrt{S/\mu T}$, $\alpha=1/T$, $\beta= 2\eta\sqrt{6\log(3k/\delta)}$ and $\gamma = \eta^2(1+\beta)^2 k^2$
\STATE $\forall~j\in [k]$ let  $w_j(1)=1$.
\FOR{$t=1,\ldots,T$}
    \STATE $\forall~j\in [k]$, let $p_j(t)=(1-\gamma)\frac{w_j(t)}{\sum_{l=1}^{k}w_l(k)}+\frac{\gamma}{k}$
    \STATE Choose action $i_t$ with probability $p_{i_t}(t)$
    \STATE Query reward $g_{j_t}(t)$ with probability $q_{j_t}(t)=\frac{\sqrt{p_{j_t}(t)}}{\sum_{j}\sqrt{p_{j}(t)}}$
    \STATE $\forall~j\in[k]$, let $\tilde{g}_{j}(t)=\frac{1}{q_{j}(t)}\left(g_{j}(t)\mathbf{1}_{j_t=j}+\beta\right)$
    \STATE $\forall~j\in [k]$, let $w_j(t+1)= w_j(t)\exp(\eta \tilde{g}_j(t))+\frac{e\alpha}{k}\sum_{i=1}^{T}w_i(t)$
\ENDFOR
\end{algorithmic}
\end{algorithm}

The following theorem, which is proven along similar lines to \thmref{thm:mab}, shows that in this setting as well, we get the same kind of dependence on the distribution vectors $\bp(t)$ as in the standard bandit setting.

\begin{theorem}\label{thm:mabswitch}
Suppose that $T$ is sufficiently large (and thus $\eta$ and $\beta$ sufficiently small) so that $(1+\beta)^2\leq 2$. Then for any $v\in [1,k]$, it holds that with probability at least $1-\delta-P(v,\delta,\mu)$ that the sequence of rewards $g_{i_1}(1),\ldots,g_{i_T}(T)$ returned by algorithm \ref{alg:banditsswitch} satisfies the following, simultaneously over all segmentations of $\{1,\ldots,T\}$ to $S$ epochs and a choice of action $i^s$ to each epoch:
\begin{align*}
&\sum_{s=1}^{S}\sum_{t=T_{s}+1}^{T_{s+1}}g_{i^s}(t)-\sum_{t=1}^{T}g_{i_t}(t)\\
&\;\;\leq~ \tilde{\Ocal}\left(\sqrt{S\left(\frac{v^2}{\mu}+\mu+v\right)T}+\frac{k^2}{\mu}
+\frac{k^2}{T^{3/2}}\right).\\
\end{align*}
The $\tilde{\Ocal}$ notation hides numerical constants and factors logarithmic in $k$ and $\delta$.
\end{theorem}

In particular, we can also get a parallel version of \thmref{thm:nonuniformupperbound}, which shows that when there are only a small number of ``good'' actions (compared to $k$), the leading term has decaying dependence on $k$, unlike standard bandit algorithms where the dependence on $k$ is always $\sqrt{k}$.

\begin{theorem}\label{thm:nonuniformupperboundswitch}
Suppose that the reward of each action is chosen i.i.d. from a distribution supported on $[0,1]$. Furthermore, suppose that at each epoch $s$, there exists a subset $G^s\subset [k]$ of actions and a parameter $\Delta>0$ (where $|G^s|,\Delta$ are considered constants independent of $k,T$), such that  the expected reward of any action in $G^s$ is larger than the expected reward of any action in $[k]\setminus G^s$ by at least $\Delta$.
Then if we run \algref{alg:banditsswitch} with
\[
\mu = k^{\min\left\{1,~\max\left\{0,~\frac{4}{3}-\frac{1}{3}\log_k(T)
\right\}\right\}},
\]
it holds with probability at least $1-\delta$ that the regret of the algorithm is at most
\[
\tilde{\Ocal}\left(\sqrt{S k^{\max\{0,\frac{4}3-\frac{1}{3}\log_k(T)\}}T}
\right),
\]
where the $\tilde{\Ocal}$ notation hides numerical constants and factors logarithmic in $\delta$ and $k$.
\end{theorem}

Now, we are ready to present the main negative result of this section, which shows that in the setting of \thmref{thm:nonuniformupperbound}, \emph{any} standard bandit algorithm cannot have a regret better than $\Omega(\sqrt{kT})$, which is significantly worse. For simplicity, we will focus on the case where $S=2$: namely, that we measure regret with respect to a single action from round $1$ till some $t_0$, and then from $t_0+1$ till $T$. Moreover, we consider a simple case where $|G^1|=|G^2|=1$ and $\Delta=1/5$, so there is just a single action at a time which is significantly better than all the other actions in expectation.

\begin{theorem}\label{thm:lowboundswitch}
Suppose that $T\geq Ck$ for some sufficiently large universal constant $C$. Then in the setting of \thmref{thm:nonuniformupperbound}, there exists a randomized reward assignment (with $|G^1|=|G^2|=1$ and $\Delta=1/5$), such that for any standard bandit algorithm, its expected regret (over the rewards assignment and the algorithm's randomness) is at least $0.007 \sqrt{(k-1)T}$.
\end{theorem}
The constant $0.007$ is rather arbitrary and is not the tightest possible.

We note that a related $\Omega(\sqrt{T})$ lower bound has been obtained in \cite{GariMou11}. However, their result does not apply to the case $S=2$ and more importantly, does not quantify a dependence on $k$. It is interesting to note that unlike the standard lower bound proof for standard bandits \cite{AuerCesFrSc02}, we obtain here an $\Omega(\sqrt{kT})$ regret even when $\Delta>0$ is fixed and doesn't decay with $T$.

\section{The Necessity of a Non-Uniform Querying Distribution}

The theoretical results above demonstrated the efficacy of our approach, compared to standard bandit algorithms. However, the exact form of our querying distribution (querying action $i$ with probability proportional to $\sqrt{p_j(t)}$) might still seem a bit mysterious. For example, maybe one can obtain similar results just by querying actions uniformly at random? Indeed, this is what has been done in some other online learning scenarios where queries were allowed (e.g., \cite{YuMa09,AgarwalDekelXiao10}). However, we show below that in the adversarial setting, an adaptive and non-uniform querying distribution is indeed necessary to obtain regret bounds better than $\sqrt{kT}$. For simplicity, we return to our basic setting, where our goal is to compete with just the best single fixed action in hindsight.

\begin{theorem}\label{thm:fixedexplorelowerbound}
Consider any online algorithm over $k>2$ actions and horizon $T$, which queries the actions based on a fixed distribution. Then there exists a strategy for the adversary conforming to the setting described in \thmref{thm:nonuniformupperbound}, for which the algorithm's regret is at least $c\sqrt{kT}$ for some universal constant $c$.
\end{theorem}
A proof sketch is presented in the appendix of the full version. The intuition of the proof is that if the querying distribution is fixed, and there are only a small number of ``good'' actions, then we spend too much time querying irrelevant actions, and this hurts our regret performance.

\section{Experiments}

We compare the decoupled approach with common multi-armed bandit algorithms in a simulated adversarial setting.
Our user chooses between $k$ communication channels, where sensing and transmission can be decoupled. In other words, she may choose a certain channel for transmission while sensing (i.e., querying) a different, seemingly less attractive, channel.

We simulate a heavily loaded UWB environment with a single, alternating, channel which is fit for transmission.
The rewards of $k-1$ channels are drawn from alternating uniform and truncated Gaussian distributions with random parameters, yielding adversarial rewards in the range $\left[0,6\right]$. The remaining channel yields stochastic rewards drawn from a truncated Gaussian distribution bounded in the same range but with a mean drawn from $\left[3,6\right]$. The identity of the better channel and its distribution parameters are re-drawn at exponentially distributed switching times.
\begin{figure}[ht]
\vskip 0in
\begin{center}
\centerline{\includegraphics[width=\columnwidth]{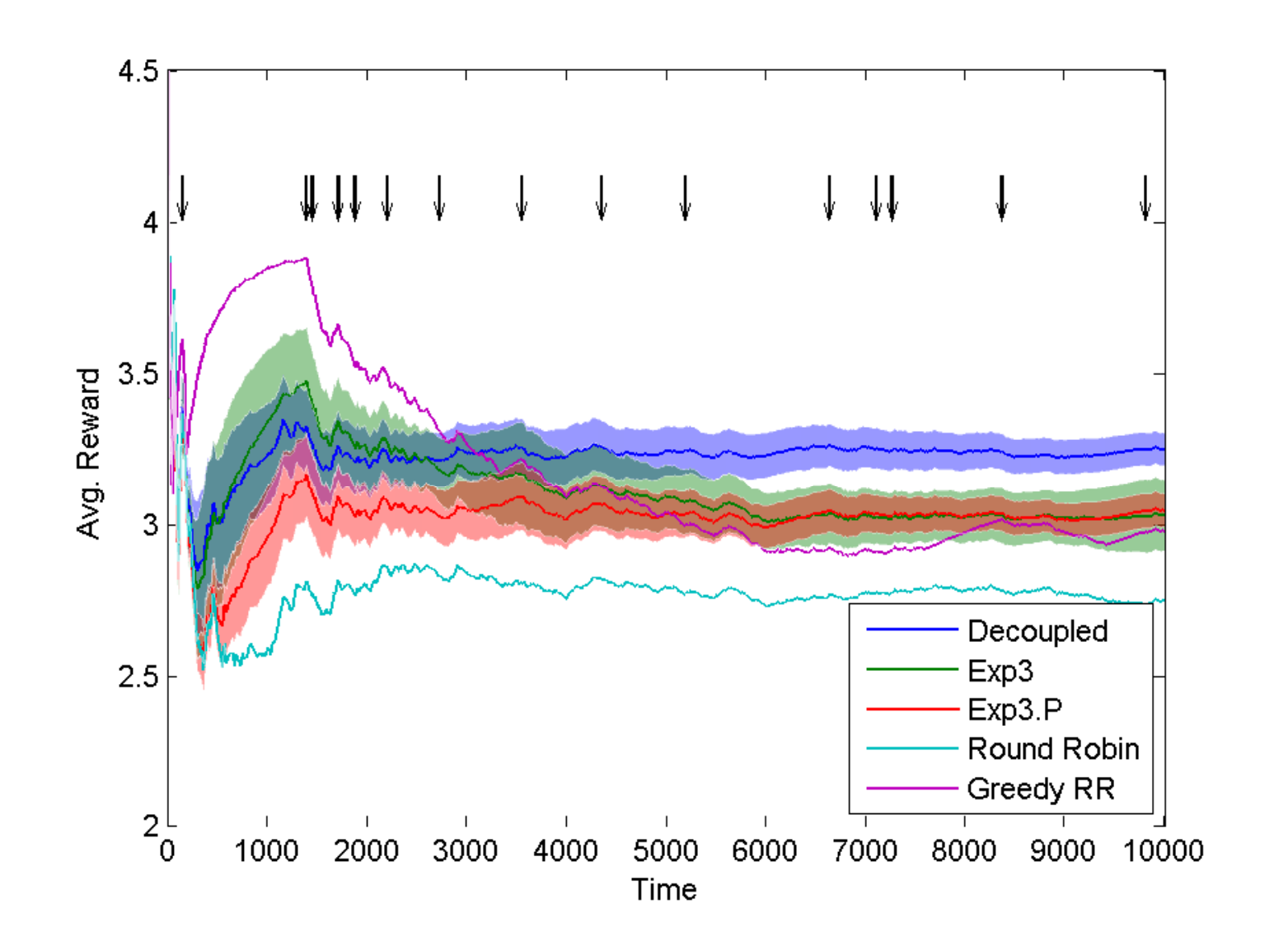}}
\vskip -0.1in
\caption{Average reward for different algorithms over time. Shaded areas around plots represent the standard deviation over repetitions.}
\label{fig:avgReward}
\end{center}
\vskip -0.3in
\end{figure}

Figure \ref{fig:avgReward} displays the results of a scenario with $k = 10$ channels, comparing the average reward acquired by the different algorithms over $T=10,000$ rounds. We implemented Algorithm \ref{alg:bandits}, Exp3 \cite{AuerCesFrSc02}, Exp3.P \cite{AuerCesFrSc02}, a simple round robin policy (which just cycles through the arms in a fixed order) and a ``greedy'' decoupled form of round robin, which performs uniform queries and picks actions greedily based on the highest empirical average reward. The black arrows indicate rounds in which the identity of the stochastic arm and its distribution parameters were re-drawn.
The results are averaged over $50$ repetitions of a specific realization of rewards. Although we have tested our algorithm's performance on several realizations of switching times and rewards with very good results, we display a single realization of these for the sake of clarity.
\begin{figure}[ht]
\vskip 0in
\begin{center}
\centerline{\includegraphics[width=0.8\columnwidth]{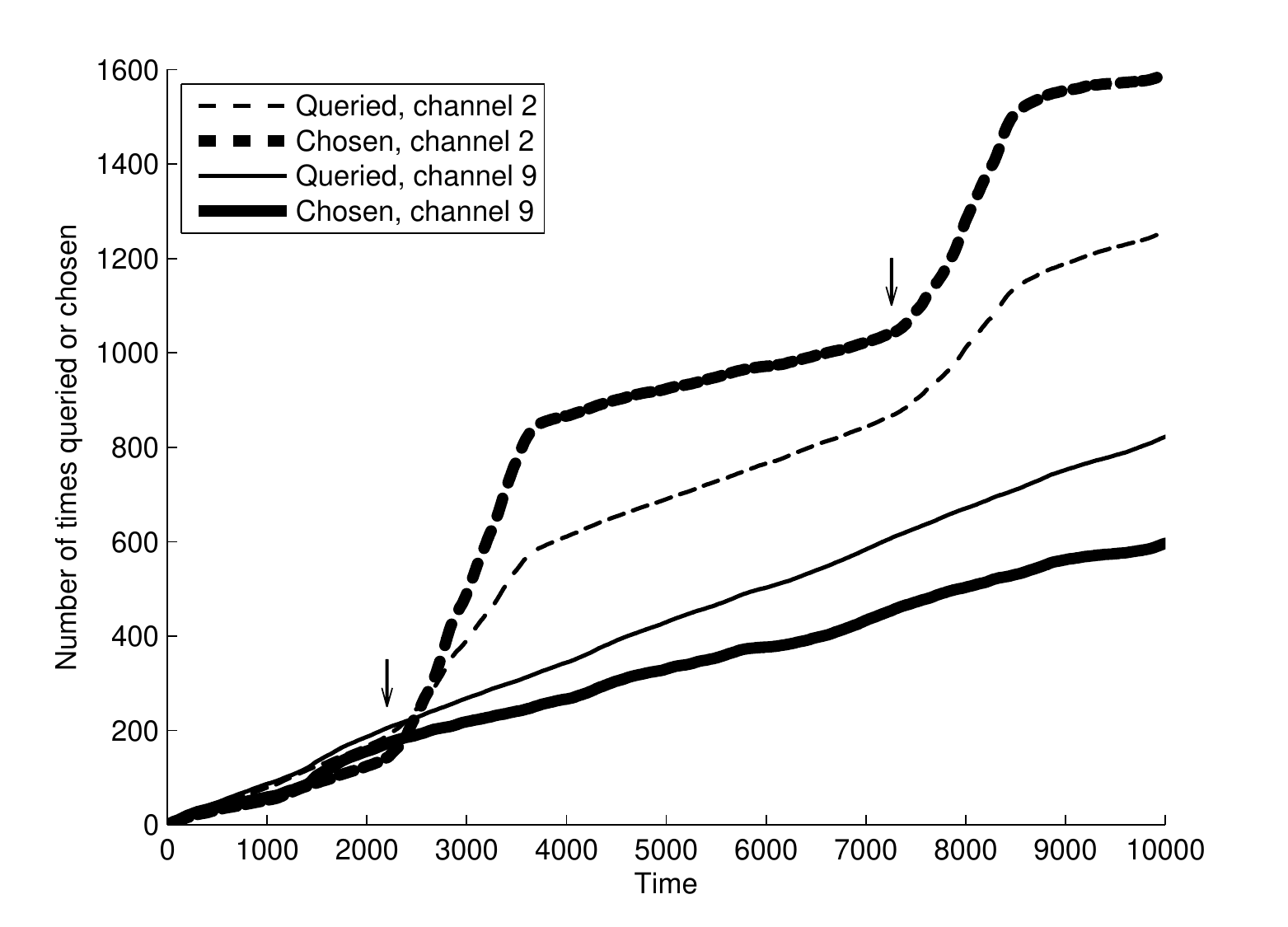}}
\vskip -0.1in
\caption{Number of times channels were chosen and queried over time, for two of $k=10$ arms. Arrows mark times in which channel 2 was drawn as the better channel.}
\label{fig:sampleQuery}
\end{center}
\vskip -0.3in
\end{figure}

Figure \ref{fig:sampleQuery} displays the dynamics of channel selection for two of the $k=10$ channels. The thick plots represent the number of times a channel was chosen over time, and the thin plots represent the number of times it was queried. The dashed plots represent a channel which was drawn as the better channel during some periods, resulting in a relatively high average reward, while the solid plots represent a channel with a low average reward. The increased flexibility of the decoupled approach is evident from the graph, as well as the adaptive, nonlinear sampling policy.

Comments: We implement Algorithm \ref{alg:bandits} and not Algorithm \ref{alg:banditsswitch} since the number of switches is unknown a-priori. Also, the rewards are in the range $\left[0,6\right]$ in order to keep all implemented algorithms on a similar scale, without violating the boundedness assumption.

\section{Discussion}\label{sec:discussion}

In this paper, we analyzed if and how one can benefit in settings where exploration and exploitation can be ``decoupled:'' namely, that one can query for rewards independently of the action actually picked. We developed some algorithms for this setting, and showed that these can indeed lead to improved results, compared to the standard bandit setting, under certain conditions. We also performed some experiments that corroborate our theoretical findings.

For simplicity, we focused on the case where only a single reward may be queried. If $c>1$ queries are allowed, it is not hard to show parallel guarantees to those in this paper, where the dependence on $k$ is replaced by dependence on $k/c$. Algorithmically, one simply needs to repeatedly sample from the query distribution $c$ times, instead of a single time. We conjecture that similar lower bounds can be obtained as well. Interestingly, it seems that being allowed to see the reward of the action actually picked, on top of the queried reward, does not result in significantly improved regret guarantees (other than better constants).

Several open questions remain. First, our results do not apply when the rewards are chosen by an adaptive adversary (namely, that the rewards are not fixed in advance but may be chosen individually at each round, based on the algorithm's behavior in previous rounds). This is not just for technical reasons, but also because data and algorithm dependent quantities like $P(v,\delta,\mu)$ do not make much sense if the rewards are not considered as fixed quantities.

A second open question concerns the possible correlation between sensing and exploration. In some applications it is plausible that the choice of which arm to exploit affects the quality of the sample of the arm that is explored. For instance, in the UWB sensing example discussed in the introduction transmitting and receiving in the same channel is much less preferred than sensing in another channel because of interference in the same frequency band. It would be interesting to model such dependence and take it into account in the learning process.

Finally, it remains to extend other bandit-related algorithms, such as EXP4 \cite{AuerCesFrSc02}, to our setting, and study the advantage of decoupling in other adversarial online learning problems.

\subsection*{Acknowledgements.}
This research was partially supported by the CORNET consortium (\url{http://www.cornet.org.il/}).

\bibliographystyle{icml2012}
\bibliography{mybib}

\appendix

\onecolumn

\section{Appendix}

\subsection{Proof of \thmref{thm:mab}}

We begin by noticing that for any possible distribution $p_1(t),\ldots,p_k(t)$, it must hold that $\norm{\bp(t)}_{1/2}\in [1,k]$. We will use this observation implicitly throughout the proof.

For notational simplicity, we will write $P(v)$ instead of $P(v,\delta,\mu)$, since we will mainly consider things as a function of $v$ where $\delta,\mu$ are fixed.

We will need the following two lemmas.

\begin{lemma}\label{lem:g}
Suppose that $\beta\leq 1$. Then it holds with probability at least $1-\delta$ that for any $i=1,\ldots,k$,
\[
\sum_{t=1}^{T}\tilde{g}_i(t)\geq \sum_{t=1}^{T}g_i(t)-\frac{\log(k/\delta)}{\beta}
\]
\end{lemma}
\begin{proof}
Let $\E_t$ denote expectation with respect to the algorithm's randomness at round $t$, conditioned on the previous rounds.
Since $\exp(x)\leq 1+x+x^2$ for $x\leq 1$, we have by definition of $\tilde{g}_i(t)$ that
\begin{align*}
&\E_t\left[\exp\left(\beta(g_{i}(t)-\tilde{g}_i(t))\right)\right]\\
&= \E_t\left[\exp\left(\beta\left( g_{i}(t)-\frac{g_i(t)\mathbf{1}_{j_t=i}}{q_i(t)}\right)-\frac{\beta^2}{q_i(t)}\right)\right]\\
&\leq
\left(1+\E_t\left[\beta\left( g_{i}(t)-\frac{g_i(t)\mathbf{1}_{j_t=i}}{q_i(t)}\right)\right]+\E_t\left[\left(\beta\left( g_{i}(t)-\frac{g_i(t)\mathbf{1}_{j_t=i}}{q_i(t)}\right)\right)^2\right]\right)
\exp\left(-\frac{\beta^2}{q_i(t)}\right)\\
&\leq
\left(1+0+\beta^2\E_t\left[\left(\frac{g_i(t)\mathbf{1}_{j_t=i}}
{q_i(t)}\right)^2\right]\right)
\exp\left(-\frac{\beta^2}{q_i(t)}\right)\\
&\leq
\left(1+\frac{\beta^2}{q_i(t)}\right)
\exp\left(-\frac{\beta^2}{q_i(t)}\right).
\end{align*}
Using the fact that $(1+x)\exp(-x)\leq 1$, we get that this expression is at most $1$. As a result, we have
\[
\E\left[\exp\left(\beta\sum_{t=1}^{T}\left(g_i(t)-\tilde{g}_i(t)\right)\right)\right]
\leq 1.
\]
Now, by a standard Chernoff technique, we know that
\[
\Pr\left(\sum_{t=1}^{T}\left(g_i(t)-\tilde{g}_i(t)\right)>\epsilon\right)
~\leq~ \exp(-\beta \epsilon)\E\left[\exp\left(\beta\sum_{t=1}^{T}\left(g_i(t)-\tilde{g}_i(t)\right)\right)\right]
~\leq~ \exp(-\beta \epsilon).
\]
Substituting $\delta=\exp(-\beta \epsilon)$, solving for $\epsilon$, and using a union bound to make the result hold simultaneously for all $i$, the result follows.
\end{proof}

We will also need the following straightforward corollary of Freedman's inequality \cite{Freedman75} (see also Lemma A.8 in \cite{CesaBianchiLu06})
\begin{lemma}\label{lem:freedman}
Let $X_1,\ldots,X_T$ be a martingale difference sequence with respect to the filtration $\{\mathcal{F}_t\}_{t=1,\ldots,T}$, and with $|X_i|\leq B$ almost surely for all $i$. Also, suppose that for some fixed $v>0$ and confidence parameter $P(v)\in(0,1)$, it holds that
$\Pr(\sum_{t=1}^{T}\E[X_t^2|\mathcal{F}_{t-1}]> vT)\leq P(v)$.
Then for any $\delta\in (0,1)$, it holds with probability at least $1-\delta-P(v)$ that
\[
\sum_{t=1}^{T}X_t \leq \sqrt{2\log\left(\frac{1}{\delta}\right)vT}+\frac{B}{2}\log\left(\frac{1}{\delta}\right).
\]
\end{lemma}

We can now turn to prove the main theorem. We define the potential function $W_t = \sum_{j=1}^{k}w_j(t)$, and get that
\begin{equation}\label{eq:pbegin}
\frac{W_{t+1}}{W_t} ~=~ \sum_{j=1}^{k}\frac{w_j(t)}{\sum_{l=1}^{k}w_l(t)}\exp(\eta \tilde{g}_j(t)).
\end{equation}
We have that $\eta\tilde{g}_j(t)\leq 1$, since by definition of the various parameters,
\[
\eta \tilde{g}_j(t) ~\leq~  \frac{\eta(1+\beta)}{q_j(t)}
~\leq~ \frac{\eta(1+\beta)}{\sqrt{\gamma/k}}\sqrt{\norm{\bp(t)}_{1/2}}
~\leq~ k\frac{\eta(1+\beta)}{\sqrt{\gamma}}~\leq~ 1.
\]

Using the definition of $p_j(t)$ and the inequality $\exp(x)\leq 1+x+x^2$ for any $x\leq 1$, we can upper bound \eqref{eq:pbegin} by
\begin{align*}
&\sum_{j=1}^{k}\frac{p_j(t)-\gamma/k}{1-\gamma}\left(1+\eta\tilde{g}_j(t)+\eta^2\tilde{g}_j(t)^2\right)\\
&\leq~ 1+\frac{\eta}{1-\gamma}\sum_{j=1}^{k}p_j(t)\tilde{g}_j(t)+\frac{\eta^2}{1-\gamma}\sum_{j=1}^{k}p_j(t)\tilde{g}_j(t)^2.
\end{align*}
Taking logarithms and using the fact that $\log(1+x)\leq x$, we get
\[
\log\left(\frac{W_{t+1}}{W_t}\right) ~\leq~ \frac{\eta}{1-\gamma}\sum_{j=1}^{k}p_j(t)\tilde{g}_j(t)+\frac{\eta^2}{1-\gamma}\sum_{j=1}^{k}p_j(t)\tilde{g}_j(t)^2.
\]
Summing over all $t$, and canceling the resulting telescopic series, we get
\begin{equation}\label{eq:pupbound}
\log\left(\frac{W_{T+1}}{W_1}\right)
~\leq~ \frac{\eta}{1-\gamma}\sum_{t=1}^{T}\sum_{j=1}^{k}p_{j}(t)\tilde{g}_j(t)
+\frac{\eta^2}{1-\gamma}\sum_{t=1}^{T}\sum_{j=1}^{k}p_j(t)\tilde{g}_j(t)^2.
\end{equation}
Also, for any fixed action $i$, we have
\begin{equation}\label{eq:plowbound}
\log\left(\frac{W_{T+1}}{W_1}\right) \geq \log\left(\frac{w_i(T+1)}{W_1}\right) = \eta\sum_{t=1}^{T}\tilde{g}_i(t)-\log(k).
\end{equation}
Combining \eqref{eq:pupbound} with \eqref{eq:plowbound} and slightly rearranging and simplifying, we get
\begin{equation}\label{eq:porbound}
\sum_{t=1}^{T}\tilde{g}_i(t) - \frac{1}{1-\gamma}\sum_{t=1}^{T}\sum_{j=1}^{k}p_{j}(t)\tilde{g}_j(t) \leq \frac{\log(k)}{\eta}+
\frac{\eta}{1-\gamma}\sum_{t=1}^{T}\sum_{j=1}^{k}p_j(t)\tilde{g}_j(t)^2.
\end{equation}

We now start to analyze the various terms in this expression. At several points in what follows, we will implicitly use the definition of $q_j(t)$ and the fact that $\norm{\bp(t)}_{1/2}\in [1,k]$.

Let $\E_t$ denote expectation with respect to the randomness of the algorithm on round $t$, conditioned on the previous rounds. Also, let
\[
g'_j(t) = \frac{g_j(t)\mathbf{1}_{j_t=j}}{q_j(t)},
\]
and note that $\tilde{g}_j(t)=g'_j(t)+\frac{\beta}{q_j(t)}$ and $\E_t[g'_j(t)] = g_j(t)$.
We have that
\begin{equation}\label{eq:pg}
\sum_{j=1}^{k}p_j(t)(\tilde{g}_j(t))~=~
\sum_{j=1}^{k}p_j(t)g'(t)+\beta\sum_{j=1}^{k}\frac{p_j(t)}{q_j(t)}
~=~\sum_{j=1}^{k}p_j(t)g'(t)+\beta\norm{\bp(t)}_{1/2}.
\end{equation}
Also, $\sum_{j=1}^{k}p_j(t)(g'(t)-g(t))$ is a martingale difference sequence (indexed by $t$), it holds that
\begin{align*}
\E_t\left[\left(\sum_{j=1}^{k}p_j(t)(g'(t)-g(t))\right)^2\right] ~&\leq~
\E_t\left[\left(\sum_{j=1}^{k}p_j(t)g'(t)\right)^2\right]
~\leq~ \sum_{r=1}^{k}q_r(t)\left(\sum_{j=1}^{k}p_j(t)\frac{\mathbf{1}_{r=j}}{q_j(t)}\right)^2\\
&=~ \sum_{r=1}^{k}\frac{p_r^2(t)}{q_r(t)}
~=~ \sqrt{\norm{\bp(t)}_{1/2}}\sum_{r=1}^{k}p_r^{3/2}(t)~\leq~
\norm{\bp(t)}_{1/2},
\end{align*}
and
\begin{align*}
\sum_{j=1}^{k}p_j(t)(g'(t)-g(t)~\leq~ \sum_{j=1}^{k}p_j(t)g'(t)
~\leq~ \max_j \frac{p_j}{q_j(t)} ~\leq~ \sqrt{\norm{\bp(t)}_{1/2}}~\leq~ \sqrt{k}.
\end{align*}
Therefore, applying \lemref{lem:freedman}, and using the assumptions stated in the theorem, it holds with probability at least $1-\delta-P(v)$ that
\begin{equation}\label{eq:conc1}
\sum_{t=1}^{T}\sum_{j=1}^{k}p_j(t)g'(t) \leq \sum_{t=1}^{T}\sum_{j=1}^{k}p_j(t)g(t)+\sqrt{2\log\left(\frac{1}{\delta}\right)vT}
+\frac{\sqrt{k}}{2}\log\left(\frac{1}{\delta}\right).
\end{equation}
Moreover, we can apply Azuma's inequality with respect to the martingale difference sequence $\sum_{j=1}^{k}p_j(t)g(t)-g_{i_t}(t)$, indexed by $t$ (since $i_t$ is chosen with probability $p_{i_t}(t)$), and get that with probability at least $1-\delta$,
\begin{equation}\label{eq:azuit}
\sum_{t=1}^{T}\sum_{j=1}^{k}p_j(t)g(t)-g_{i_t}(t)\leq \sqrt{\frac{1}{2}\log\left(\frac{1}{\delta}\right)T}.
\end{equation}

Combining \eqref{eq:pg}, \eqref{eq:conc1} and \eqref{eq:azuit} with a union bound, and recalling that the event $\sum_{t=1}^{T}\norm{\bp(t)}_{1/2}\leq vT$ is assumed to hold with probability at least $1-P(v)$, we get that with probability at least $1-2\delta-P(v)$,
\begin{equation}\label{eq:pgfinal}
\sum_{t=1}^{T}\sum_{j=1}^{k}p_j(t)(\tilde{g}_j(t))-\sum_{t=1}^{T}g_{i_t}(t)
\leq \beta vT+
\sqrt{2\log\left(\frac{1}{\delta}\right)vT}+
\sqrt{\frac{1}{2}\log\left(\frac{1}{\delta}\right)T}
+\frac{\sqrt{k}}{2}\log\left(\frac{1}{\delta}\right).
\end{equation}

We now turn to analyze the term $\sum_{j=1}^{k}p_j(t)\tilde{g}_j^2(t)$, using substantially the same approach. We have that
\begin{align*}
\sum_{j=1}^{k}p_j(t)\tilde{g}_j^2(t) &~=~
\sum_{j=1}^{k}p_j(t)\left(g'_j(t)+\frac{\beta}{q_j(t)}\right)^2
~\leq~ 2\sum_{j=1}^{k}p_j(t)g'^2_j(t)+2\beta^2 \sum_{j=1}^{k}\frac{p_j(t)}{q_j^2(t)}
~\leq~
2\sum_{j=1}^{k}p_j(t)g'^2_j(t)+2\beta^2 k^2.
\end{align*}
We note that $\sum_{j=1}^{k}p_j(t)g'^2(t) ~\leq~ \max_j \frac{p_j(t)}{q_j^2(t)} = \norm{\bp(t)}_{1/2} \leq k$. This implies that
\[
\sum_{t=1}^{T} \E_t\left[\left(\sum_{j=1}^{k}p_j(t)g'^2(t)\right)^2\right]~\leq~ \sum_{t=1}^{T} \norm{\bp(t)}_{1/2}^2 ~\leq~ \left(\sum_{t=1}^{T}\norm{\bp(t)}_{1/2}\right)^2 ~\leq~ (vT)^2 .
\]
Applying \lemref{lem:freedman}, we get that with probability at least $1-\delta-P(v)$,
\[
\sum_{t=1}^{T}\sum_{j=1}^{k}p_j(t)g'^2_j(t) ~\leq~ \sum_{t=1}^{T}\E_t\left[\sum_{j=1}^{k}p_j(t)g'^2_j(t)\right]
+vT\sqrt{2\log\left(\frac{1}{\delta}\right)}
+\frac{k}{2}\log\left(\frac{1}{\delta}\right).
\]
Moreover,
\[
\E_t\left[\sum_{j=1}^{k}p_j(t)g'^2_j(t)\right] ~\leq~ \sum_{r=1}^{k}q_r(t)\frac{p_r(t)}{q_r^2(t)}~=~\norm{\bp(t)}_{1/2},
\]
so overall, we get that with probability at least $1-\delta-P(v)$,
\begin{equation}\label{eq:pg2final}
\sum_{t=1}^{T}\sum_{j=1}^{k}p_j(t)\tilde{g}^2_j(t)
~\leq~
2vT\left(1+\sqrt{2\log\left(\frac{1}{\delta}\right)}\right)
+k\log\left(\frac{1}{\delta}\right)+2\beta^2 k^2.
\end{equation}
Combining \lemref{lem:g}, \eqref{eq:pgfinal} and \eqref{eq:pg2final} with a union bound, substituting into \eqref{eq:porbound}, and somewhat simplifying, we get that with probability at least $1-\delta-P(v)$,
\begin{align*}
\max_{i}\sum_{t=1}^{T}g_i(t)-\sum_{t=1}^{T}g_{i_t}(t)
&~\leq~ \gamma T+2vT\left(\beta+2\eta\sqrt{6\log(3/\delta)}\right)
+2\sqrt{5\log(3/\delta)vT}+\frac{\log(3k/\delta)}{\beta}+\frac{\log(k)}{\eta}\
\\
&+\tilde{\Ocal}\left(\sqrt{k}+\eta k + k^2\beta^2 \eta\right).
\end{align*}
where $\tilde{\Ocal}$ hides numerical constants and factors logarithmic in $\delta$.
Substituting our choices of $\gamma$ and $\beta$, and again somewhat simplifying, we get the bound
\begin{align*}
\max_{i}\sum_{t=1}^{T}g_i(t)-\sum_{t=1}^{T}g_{i_t}(t)
~\leq~ &\eta\left(8\sqrt{6\log\left(\frac{3k}{\delta}\right)}vT\right)
+\frac{1}{\eta}\left(\sqrt{\frac{1}{24}\log\left(\frac{3k}{\delta}\right)}
+\log(k)\right)+2\eta^2 k^2 T \\&+ 2\sqrt{5\log(3/\delta)vT}+
\tilde{\Ocal}\left(\sqrt{k}+\eta k + k^2\eta^3\right).
\end{align*}
Plugging in $\eta=\sqrt{1/\mu T}$, we get the bound stated in the theorem.

\subsection{Proof of \thmref{thm:nonuniformupperbound}}

For notational simplicity, we will use the $\Ocal$-notation to hide both constants and second-order factors (as $T/k\rightarrow \infty$).
Inspecting the proof of \thmref{thm:mab}, it is easy to verify\footnote{The difference from \thmref{thm:mab} is that the term $\sum_{t=1}^{T}g_{i_t}(t)$ is replaced by $\sum_{i=1}^{k}p_i(t)g_{i}(t)$. In the proof, we transformed the latter to the former by a martingale argument, but we could have just left it there and achieve the same bound.} that it implies that with probability at least $1-\delta-P(v,\delta,\mu)$,
\[
\max_{i\in[k]} \sum_{t=1}^{T}g_i(t)-\sum_{t=1}^{T}\sum_{j=1}^{k}p_j(t)g_{j}(t)~\leq~
\tilde{\Ocal}\left(\sqrt{\left(\frac{v^2}{\mu}+\mu+v\right)T}+\frac{k^2}{\mu}
+\frac{k^2}{T^{3/2}}\right).
\]
Suppose w.l.o.g. action $1$ is in $G$. Then it follows that
\[
\sum_{t=1}^{T}\sum_{j=1}^{k}p_j(t)\left(g_1(t)-g_j(t)\right)~\leq~
\tilde{\Ocal}\left(\sqrt{\left(\frac{v^2}{\mu}+\mu+v\right)T}+\frac{k^2}{\mu}
+\frac{k^2}{T^{3/2}}\right).
\]
This bound holds for any choice of rewards. Now, we note that each $g_j(t)$ is chosen i.i.d. and independently of $p_j(t)$), and thus $\sum_{j=1}^{k}p_j(t)((g_1(t)-g_j(t))-\E[g_1(t)-g_j(t)])$ is a martingale difference sequence. Applying Azuma's inequality, we get that with probability at least $1-\delta$ over the choice of rewards,
\begin{align*}
&\sum_{t=1}^{T}\sum_{j=1}^{k}p_j(t)\left(g_1(t)-g_j(t)\right) ~\geq~
\sum_{t=1}^{T}\sum_{j=1}^{k}p_j(t)\left(\E[g_1(t)-g_j(t)]\right) - \sqrt{2\log(1/\delta)T}\\
&\geq~ \sum_{t=1}^{T}\sum_{j\in[k]\setminus G}p_j(t)\Delta - \sqrt{2\log(1/\delta)T}.
\end{align*}
Thus, by a union bound, with probability at least $1-2\delta-P(v,\delta,\mu)$ over the randomness of the rewards and the algorithm, we get
\begin{equation}\label{eq:prel}
\sum_{t=1}^{T}\sum_{j\in[k]\setminus G}p_j(t) \leq \tilde{\Ocal}\left(\sqrt{\left(\frac{v^2}{\mu}+\mu+v\right)T}+\frac{k^2}{\mu}
+\frac{k^2}{T^{3/2}}\right),
\end{equation}
where $\tilde{\Ocal}$ hides an inverse dependence on $\Delta$.
Now, we relate the left hand size to $\frac{1}{T}\sum_{t=1}^{T}\norm{\bp(t)}_{1/2}$. To do so, we note that for any vector $\bx$ with support of size $|G|$, it holds that $\norm{\bx}_{1/2}\leq |G|\norm{\bx}_{1}$.
Using this and the fact that $(a+b)^2\leq 2a^2+2b^2$, we have
\begin{align*}
&\norm{\bp(t)}_{1/2} ~=~ \left(\sum_{j\in G}\sqrt{p_j(t)}+\sum_{j\in [k]\setminus G}\sqrt{p_j(t)}\right)^2
~\leq~ 2\left(\sum_{j\in G}\sqrt{p_j(t)}\right)^2+2\left(\sum_{j\in [k]\setminus G}\sqrt{p_j(t)}\right)^2\\
&\leq~ 2|G|+k \sum_{j\in [k]\setminus G}p_j(t).
\end{align*}
Plugging this back to \eqref{eq:prel}, and recalling that $|G|$ is considered a constant independent of $k,T$, we get that with probability at least $1-2\delta-P(v,\delta,\mu)$, it holds that
\[
\frac{1}{T}\sum_{t=1}^{T}\norm{\bp(t)}_{1/2} \leq \tilde{\Ocal}
\left(\sqrt{\left(\frac{v^2}{\mu}+\mu+v\right)\frac{k^2}{T}}+\frac{k^3}{\mu T}
+\frac{k^3}{T^{5/2}}\right).
\]
Recall that this bound holds for any $v$. In particular, if we pick $v=k$, then $P(v,\delta,\mu)=0$, and we get that with probability at least $1-2\delta$,
\begin{equation}\label{eq:rec}
\frac{1}{T}\sum_{t=1}^{T}\norm{\bp(t)}_{1/2} \leq \tilde{\Ocal}
\left(\frac{k^2}{\sqrt{\mu T}}+\frac{k^3}{\mu T}
+\frac{k^3}{T^{5/2}}\right).
\end{equation}
This gives us a high-probability bound, holding with probability at least $1-2\delta$, on $\frac{1}{T}\sum_{t=1}^{T}\norm{\bp(t)}_{1/2}$. But this means that if we pick $v$ to equal the right hand size of \eqref{eq:rec}, then by the very definition of $P(v,\delta,\mu)$, we get $P(v,\delta,\mu)=2\delta$. Using this choice of $v$ and applying \thmref{thm:mab}, it follows that with probability at least $1-4\delta$, the regret obtained by the algorithm is at most
\begin{equation}\label{eq:intbound}
\tilde{\Ocal}\left(\sqrt{\left(\frac{v^2}{\mu}+\mu+v\right)T}+\frac{k^2}{\mu}
+\frac{k^2}{T^{3/2}}\right)~~~\text{where}~~~
v = \max\left\{1~,~\tilde{\Ocal}
\left(\frac{k^2}{\sqrt{\mu T}}+\frac{k^3}{\mu T}
+\frac{k^3}{T^{5/2}}\right)\right\}.
\end{equation}
Now, it remains to optimize over $\mu$ to get a final bound. As a sanity check, we note that when $\mu=k$ and $T\geq k$ we get
\[
v = \tilde{\Ocal}\left(\sqrt{\frac{k^3}{T}}+\frac{k^2}{T}+\frac{k^3}{T^{5/2}}\right)
\leq \tilde{\Ocal}(k),
\]
and a regret bound of $\tilde{\Ocal}(\sqrt{kT})$, same as a standard bandit algorithm. On the other hand, when $\mu=1$ and $T\geq \tilde{\Omega}(k^4)$, we get $v=\tilde{\Ocal}(1)$ and a regret bound of $\tilde{\Ocal}(\sqrt{T})$, which is much better. The caveat is that we need $T$ to be sufficiently large compared to $k$ in order to get this effect. To understand what happens in between, it will be useful to represent this bound a bit differently. Let $\alpha=\log_{T}(k)\in (0,1]$, so that $k=T^{\alpha}$, and let $\mu=T^{\beta}$ (where we need to ensure that $\beta\in [0,\alpha]$, as $\mu\in[1,k]$). Then, a rather tedious but straightforward calculation shows that the regret bound above equals
\begin{align*}
\tilde{\Ocal}\left(T^{\frac{1-\beta}{2}}+T^{2\alpha-\beta}+T^{3\alpha-\frac{3\beta+1}{2}}+T^{3\alpha-\frac{\beta}{2}-2}
+T^{\frac{1+\beta}{2}}+T^{1/2}+T^{\alpha+\frac{1-\beta}{4}}
+T^{\frac{3\alpha-\beta}{2}}+T^{\frac{3}{2}\alpha-\frac{3}{4}}
+T^{2\alpha-\frac{3}{2}}\right).
\end{align*}
Using the fact that $\beta\leq \alpha\leq 1$, we can drop the $T^{\frac{1-\beta}{2}}+T^{1/2}$ terms, since it is always dominated by the $T^{\frac{1+\beta}{2}}$ term in the expression. The same goes for the $T^{\frac{3}{2}\alpha-\frac{3}{4}}+T^{2\alpha-\frac{3}{2}}$ terms, since they are dominated by the $T^{\frac{3\alpha-\beta}{2}}$ term (as $\beta\leq\alpha\leq 1$). This also holds for the $T^{\frac{3\alpha-\beta}{2}}$ term, which is dominated by the $T^{2\alpha-\beta}$ term, and the $T^{3\alpha-\frac{\beta}{2}-2}$ term, which is dominated by the $T^{2\alpha-\beta}$ term. Thus, we now need to find the $\beta$ minimizing the maximum exponent, i.e.,
\[
\min_{\beta}\max\left\{2\alpha-\beta,
3\alpha-\frac{3\beta+1}{2},
\frac{1+\beta}{2},\alpha+\frac{1-\beta}{4}\right\}.
\]
This expression can be shown to be optimized for $\beta=\frac{1}{3}\max\{0,4\alpha-1\}$, where it equals
$\frac{1}{2}+\frac{1}{6}\max\{0,4\alpha-1\}$. Substituting back $\alpha=\log_{k}(T)$, we get the regret bound
\[
\tilde{\Ocal}\left(T^{\frac{1}{2}+\frac{1}{6}\max\{0,4\alpha-1\}}\right)
~=~ \tilde{\Ocal}\left(T^{\frac{1}{2}
+\frac{\alpha}{6}\max\left\{0,4-\frac{1}{\alpha}\right\}}\right)
~=~
\tilde{\Ocal}\left(T^{\frac{1}{2}}
k^{\frac{1}{6}\max\left\{0,4-\frac{1}{\alpha}\right\}}\right)
~=~
\tilde{\Ocal}\left(\sqrt{k^{\max\left\{0,\frac{4}{3}-\frac{1}{3}\log_{k}(T)\right\}}
T}\right),
\]
obtained using
\[
\mu = T^{\beta} = T^{\frac{1}{3}\max\{0,4\alpha-1\}} =
T^{\frac{\alpha}{3}\max\left\{0,4-\frac{1}{\alpha}\right\}}
= k^{\max\left\{0,\frac{4}{3}-\frac{1}{3}\log_k(T)\right\}}.
\]
The derivation above assumed that $\alpha\leq 1$ (namely that $T\geq k$). For $T\leq k$, we need to clip $\mu$ to be at most $k$, and the regret bound obtained above is vacuous, as it is then larger than order of $\sqrt{kT}\geq T$. Thus, the bound we have obtained holds for any relation between $k,T$.

\subsection{Proof of \thmref{thm:mabswitch}}

The proof is very similar to the one of \thmref{thm:mab}, and we will therefore skip the derivation of some steps which are identical.

We define the potential function $W_t = \sum_{j=1}^{k}w_j(t)$, and get that
\[
\frac{W_{t+1}}{W_t} ~=~ \sum_{j=1}^{k}\frac{w_j(t)}{\sum_{l=1}^{k}w_l(t)}\exp(\eta \tilde{g}_j(t))+e\alpha.
\]
Using a similar derivation as in the proof of \thmref{thm:mab}, we get
\[
\log\left(\frac{W_{t+1}}{W_t}\right) ~\leq~ \frac{\eta}{1-\gamma}\sum_{j=1}^{k}p_j(t)\tilde{g}_j(t)+\frac{\eta^2}{1-\gamma}\sum_{j=1}^{k}p_j(t)\tilde{g}_j(t)^2
+e\alpha
\]
Summing over $t=T_{s}+1,\ldots,T_{s+1}$, we get
\begin{equation}\label{eq:pupboundswitch}
\log\left(\frac{W_{T_{s+1}+1}}{W_{T_{s}+1}}\right)
~\leq~ \frac{\eta}{1-\gamma}\sum_{t=T_{s+1}}^{T_{s}+1}\sum_{j=1}^{k}p_{j}(t)\tilde{g}_j(t)
+\frac{\eta^2}{1-\gamma}\sum_{t=T_{s+1}}^{T_{s}+1}\sum_{j=1}^{k}p_j(t)\tilde{g}_j(t)^2.
\end{equation}
Now, for any fixed action $i^s$, we have
\begin{align*}
w_i(T_{s+1}+1) ~&\geq~ w_i(T_{s}+2)\exp\left(\eta\sum_{t=T_{s}+2}^{T_{s+1}}\tilde{g}_{i^s}(t)\right)\\
&\geq~ \frac{e\alpha}{k}W_{T_{s}+1}\exp\left(\eta\sum_{t=T_{s}+2}^{T_{s+1}}\tilde{g}_{i^s}(t)\right)\\
&\geq~
\frac{\alpha}{k}W_{T_{s}+1}\exp\left(\eta\sum_{t=T_{s}+1}^{T_{s+1}}\tilde{g}_{i^s}(t)\right),
\end{align*}
where in the last step we used the fact that by our parameter choices, $\eta \tilde{g}_{i^s}(t)\leq 1$ (see proof of \thmref{thm:mab}). Therefore, we get that
\begin{equation}\label{eq:plowboundswitch}
\log\left(\frac{W_{T_{s+1}+1}}{W_{T_{s}+1}}\right) \geq \log\left(\frac{w_{i_s}(T_{s+1}+1)}{W_{T_{s}+1}}\right) \geq \eta\sum_{t=T_{s}+1}^{T_{s+1}}\tilde{g}_{i^s}(t)+\log(\alpha/k).
\end{equation}
Combining \eqref{eq:pupboundswitch} with \eqref{eq:plowboundswitch} and slightly rearranging and simplifying, we get
\[
\sum_{t=T_{s}+1}^{T_{s+1}}\tilde{g}_{i^s}(t) - \frac{1}{1-\gamma}\sum_{t=1}^{T}\sum_{j=1}^{k}p_{j}(t)\tilde{g}_j(t) \leq \frac{\log(k/\alpha)S}{\eta}+
\frac{\eta}{1-\gamma}\sum_{t=T_{s}+1}^{T_{s+1}}\sum_{j=1}^{k}p_j(t)\tilde{g}_j(t)^2
+\frac{e\alpha(T_{s+1}-T_{s})}{\eta}.
\]
Summing over all time periods $s$, we get overall
\[
\sum_{s=1}^{S}\sum_{t=T_{s}+1}^{T_{s+1}}\tilde{g}_{i^s}(t) - \frac{1}{1-\gamma}\sum_{t=1}^{T}\sum_{j=1}^{k}p_{j}(t)\tilde{g}_j(t) \leq \frac{\log(k/\alpha)S}{\eta}+
\frac{\eta}{1-\gamma}\sum_{t=1}^{T}\sum_{j=1}^{k}p_j(t)\tilde{g}_j(t)^2
+\frac{e\alpha T}{\eta}.
\]
In the proof of \thmref{thm:mab}, we have already provided an analysis of these terms, which is not affected by the modification in the algorithm. Using this analysis, we end up with the following bound, holding with probability at least $1-\delta-P(v,\eta,\mu)$:
\begin{align*}
\sum_{s=1}^{S}\sum_{t=T_{s}+1}^{T_{s+1}}g_{i^s}(t)-\sum_{t=1}^{T}g_{i_t}(t)
~\leq~ &\eta\left(8\sqrt{6\log\left(\frac{3k}{\delta}\right)}vT\right)
+\frac{1}{\eta}\sqrt{\frac{1}{24}\log\left(\frac{3k}{\delta}\right)}+2\eta^2 k^2 T \\
&+2\sqrt{5\log(3/\delta)vT}+\frac{\log(k/\alpha)S+e\alpha T}{\eta}+
\tilde{\Ocal}\left(\sqrt{k}+\eta k + k^2\eta^3\right),
\end{align*}
where $\tilde{\Ocal}$ hides numerical constants and factors logarithmic in $\delta$. Plugging in $\alpha = 1/T$ and $\eta=\sqrt{S/\mu T}$, we get the bound stated in the theorem.

\subsection{Proof of \thmref{thm:nonuniformupperboundswitch}}

The proof is almost identical to the one of \thmref{thm:nonuniformupperbound}, and we will only point out the differences.

Starting in the same way, the analysis leads to the following bound:
\[
\sum_{s=1}^{S}\sum_{t={T_{s}+1}}^{T_{s+1}}\sum_{j=1}^{k}p_j(t)\left(g_{i^s}(t)-g_j(t)\right)~\leq~
\tilde{\Ocal}\left(\sqrt{S\left(\frac{v^2}{\mu}+\mu+v\right)T}+\frac{k^2}{\mu}
+\frac{k^2}{T^{3/2}}\right)
\]
This bound holds for any choice of rewards. Since each $g_{j}(t)$ is chosen i.i.d. and independently of $p_j(t)$), we get that $\sum_{j=1}^{k}p_j(t)((g_{i^s}(t)-g_j(t))-\E[g_{i^s}(t)-g_j(t)])$ is a martingale difference sequence. Applying Azuma's inequality, we get that with probability at least $1-\delta$ over the choice of rewards,
\begin{align*}
\sum_{s=1}^{S}\sum_{t={T_{s}+1}}^{T_{s+1}}\sum_{j=1}^{k} & p_j(t)\left(g_{i^s}(t)-g_j(t)\right)\\
&\geq~
\sum_{s=1}^{S}\sum_{t={T_{s}+1}}^{T_{s+1}}\sum_{j=1}^{k}p_j(t)\left(\E[g_{i^s}(t)-g_j(t)]\right) - \sqrt{2\log(1/\delta)T}\\
&\geq~ \sum_{s=1}^{S}\sum_{t={T_{s}+1}}^{T_{s+1}}\sum_{j\in[k]\setminus G^s}p_j(t)\Delta - \sqrt{2\log(1/\delta)T}.
\end{align*}
Thus, by a union bound, with probability at least $1-2\delta-P(v,\delta,\mu)$ over the randomness of the rewards and the algorithm, we get
\[
\sum_{s=1}^{S}\sum_{t={T_{s}+1}}^{T_{s+1}}\sum_{j\in[k]\setminus G^s}p_j(t) \leq \tilde{\Ocal}\left(\sqrt{\left(\frac{v^2}{\mu}+\mu+v\right)T}+\frac{k^2}{\mu}
+\frac{k^2}{T^{3/2}}\right)
\]
As in the proof of \thmref{thm:nonuniformupperbound}, we use the inequality $\norm{\bp(t)}_{1/2} \leq 2|G^s|+k \sum_{j\in [k]\setminus G}p_j(t)$ and the assumption that $|G^s|$ is considered a constant independent of $k,T$, to get
\[
\frac{1}{T}\sum_{t=1}^{T}\norm{\bp(t)}_{1/2} \leq \tilde{\Ocal}
\left(\sqrt{S\left(\frac{v^2}{\mu}+\mu+v\right)\frac{k^2}{T}}+\frac{k^3}{\mu T}
+\frac{k^3}{T^{5/2}}\right).
\]
The rest of the proof now follows verbatim the one of \thmref{thm:nonuniformupperboundswitch}, with the only difference being the addition of the $S$ factor in the square root.

\subsection{Proof of \thmref{thm:lowboundswitch}}

Following standard lower-bound proofs for multi-armed bandits, we will focus on deterministic algorithms,  We will show that there exists a randomized adversarial strategy, such that for any deterministic algorithm, the expected regret is lower bounded by $\Omega(\sqrt{kT})$. Since this bound holds for any deterministic algorithm, it also holds for randomized algorithms, which choose the action probabilistically (this is because any such algorithm can be seen as a randomization over deterministic algorithms).

The proof is inspired by the lower bound result\footnote{This result also lower bounds the achievable regret in a setting quite similar to ours. However, the construction is different and more importantly, it does not quantify the dependence on $k$, the number of actions.} of \cite{GariMou11}. We consider the following random adversary strategy.
The adversary first fixes $\Delta=1/5$. It then chooses an action $a\in \{2,\ldots,k\}$ uniformly at random, and an action $t_0\in[t]$ with probability
\[
\Pr(t_0=T) = \frac{1}{2}~~\text{and}~~\Pr(t_0=t)=\frac{1}{2(T-1)}~
\forall t\neq T~
\]
The adversary then randomly assigns i.i.d. rewards as follows (where we let $\Bcal(p)$ denote a Bernoulli distribution with parameter $p$, which takes a value of $1$ with probability $p$ and $0$ otherwise):
\[
g_i(t) \sim
\begin{cases}\Bcal\left(\frac{1}{2}\right)& i=1\\
            \Bcal\left(\frac{1}{2}-\Delta\right)& i\in [k]\setminus\{1,a\}\\
            \Bcal\left(\frac{1}{2}-\Delta\right)& i=a,t\leq t_0\\
            \Bcal\left(\frac{1}{2}+\Delta\right)& i=a,t>t_0
\end{cases}
\]
In words, action $1$ is the best action in expectation for the first $t_0$ rounds (all other actions being statistically identical), and then a randomly selected action $a$ becomes better. Also, with probability $1/2$, we have $t_0=T$, and then the distribution does not change at all. Note that both $t_0$ and $a$ are selected randomly and are not known to the learner in advance.

For the proof, we will need some notation. We let $\E[\cdot]$ denote expectation with respect to the random adversary strategy mentioned above. Also, we let $\E_{t_0}^{a}[\cdot]$ denote expectation over the adversary strategy, conditioned on the adversary picking action $a\in[k]$ and shift point $t_0\in\{1,\ldots,T\}$. In particular, we let $\E_{T}$ denote expectation over the adversary strategy, conditioned on the adversary picking $t_0=T$ (which by definition of $t_0$, implies that the reward distribution remains the same across all rounds, and the additional choice of the action $a$ does not matter). Finally, define the random variable $N_{t}^{a}$ to be the number of times the algorithm chooses action $a$, in the time window $\{t,t+1,\ldots,\min\{T,t+d\lceil \sqrt{T}\rceil\}$, where $d$ is a positive integer to be determined later.

Let us fix some $t_0<T$ and some action $a>1$. Let $\Pcal_{t_0}^{a}$ denote the probability distribution over the sequence of $d\lceil\sqrt{T}\rceil$ rewards observed by the algorithm at time steps $t_0+1,\ldots,t_0+d\lceil\sqrt{T}\rceil$, conditioned on the adversary picking action $a$ and shift point $t_0$. Also, let $\Pcal_{T}$ denote the probability distribution over such a sequence, conditioned on the adversary picking $t_0=T$ and no distribution shift occurring. Then we have the following bound on the Kullback-Leibler divergence between the two distributions:
\begin{align*}
D_{kl}\left(\Pcal_{\emptyset}||\Pcal_{t_0}^{a}\right)
~&=~ \sum_{t=t_0}^{t_0+d\lceil\sqrt{T}\rceil}D_{kl}\left(\Pcal_{\emptyset}(g_{i_t}(t)~|~g_{i_{t_0}}(t_0),\ldots,
g_{i_{t-1}(t-1)})~||~\Pcal_{t_0}^{a}(\cdot|g_{i_{t_0}}(t_0),\ldots,
g_{i_{t-1}(t-1)})\right)\\
&=~ \sum_{t=t_0}^{t_0+d\lceil\sqrt{T}\rceil}\Pcal_{\emptyset}(i_t=a)D_{kl}\left(\frac{1}{2}-\Delta,\frac{1}{2}+\Delta\right)\\
&=~
\E_{\emptyset}\left[N_{t_0}^{a}\right]2\Delta
\log\left(\frac{1+2\Delta}{1-2\Delta}\right)
~\leq~ 2\Delta\E_{\emptyset}\left[N_{t_0}^{a}\right].
\end{align*}

Using a standard information-theoretic argument, based on Pinsker's inequality (see \cite{GariMou11}, as well as Theorem 5.1 in \cite{AuerCesFrSc02}), we have that for any function $f(\br)$ of the reward sequence $\br$, whose range is at most $[0,b]$, it holds that
\[
\E_{t_0}^{a}[f(\br)]-\E_{\emptyset}[f(\br)]]\leq b
\sqrt{\frac{1}{2}D_{kl}\left(\Pcal_{\emptyset}||\Pcal_{t_0}^{a}\right)}.
\]
In particular, applying this to $N_{t_0}^{a}$, we get
\[
\E_{t_0}^{a}\left[N_{t_0}^{a}\right] \leq \E_T\left[N_{t_0}^{a}\right]
+d\lceil\sqrt{T}\rceil\sqrt{\Delta\E_T\left[N_{t_0}^{a}\right]}.
\]
Averaging over all $t_0\in\{1,\ldots,T-d\lceil\sqrt{T}\rceil\},a\in\{2,\ldots,k\}$ and applying Jensen's inequality, we get
\[
\frac{\sum_{a=2}^{k}\sum_{t_0=1}^
{T-d\lceil\sqrt{T}\rceil}
\E_{t_0}^{a}\left[N_{t_0}^{a}\right]}{(k-1)(T-d\lceil\sqrt{T}\rceil)}
\leq
\frac{\E_T\left[\sum_{a=2}^{k}\sum_{t_0=1}^{T-d\lceil\sqrt{T}\rceil}
N_{t_0}^{a}\right]}{(k-1)(T-d\lceil\sqrt{T}\rceil)}
+d\lceil\sqrt{T}\rceil\sqrt{
\frac{\Delta\E_T\left[\sum_{a=2}^{k}\sum_{t_0=1}^{T-d\lceil\sqrt{T}\rceil}
N_{t_0}^{a}\right]}{(k-1)(T-d\lceil\sqrt{T}\rceil)}}.
\]
Now, let $N^{>1}$ denote the total number of times the algorithm chooses an action in $\{2,\ldots,k\}$. It is easily seen that
\[
\sum_{a=2}^{k}\sum_{t_0=1}^{T-d\lceil\sqrt{T}\rceil}N_{t_0}^{a}
\leq d\lceil\sqrt{T}\rceil N^{>1},
\]
because on the left hand side we count every single choice of an action $>1$ at most $d\lceil\sqrt{T}\rceil$ times. Plugging it back and slightly simplifying, we get
\begin{equation}\label{eq:brelate}
\frac{\sum_{a=2}^{k}\sum_{t_0=1}^
{T-d\lceil\sqrt{T}\rceil}
\E_{t_0}^{a}\left[N_{t_0}^{a}\right]}{(k-1)(T-d\lceil\sqrt{T}\rceil)}
\leq
\frac{d\lceil\sqrt{T}\rceil}{(k-1)(T-d\lceil\sqrt{T}\rceil)}
\E_T\left[N^{>1}\right]
+\sqrt{
\frac{d^3\Delta\lceil\sqrt{T}\rceil^3}
{(k-1)(T-d\lceil\sqrt{T}\rceil)}\E_T\left[N^{>1}\right]}.
\end{equation}
The left hand side of the expression above can be interpreted as the expected number of pulls of the best action in the time window $[t_0,\ldots,t_0+d\lceil\sqrt{T}\rceil]$, conditioned on the adversary choosing $t_0\leq T-d\lceil\sqrt{T}\rceil$. Also, $\Delta\E_T[N^{>1}]$  is clearly a lower bound on the regret, if the adversary chose $t_0=T$ and action $1$ remains the best throughout all rounds. Thus, denoting the regret by $R$, we have
\begin{align*}
\E[R] &\geq \Pr\left(t_0\leq T-d\lceil\sqrt{T}\rceil\right)\E\left[R\middle|t_0\leq T-d\lceil\sqrt{T}\rceil\right]+\Pr(t_0=T)\E\left[R\middle|t_0=T\right]\\
&\geq
\frac{T-d\lceil\sqrt{T}\rceil}{2(T-1)}\E\left[R\middle|t_0\leq T-d\lceil\sqrt{T}\rceil\right]+\frac{1}{2}\E_T[R]\\
&\geq \frac{T-d\lceil\sqrt{T}\rceil}{2(T-1)}\Delta\left(d\sqrt{T}-\frac{\sum_{a=2}^{k}\sum_{t_0=1}^
{T-d\lceil\sqrt{T}\rceil}
\E_{t_0}^{a}\left[N_{t_0}^{a}\right]}{(k-1)(T-d\lceil\sqrt{T}\rceil)}
\right)+\frac{\Delta}{2}\E_T[N^{>1}].
\end{align*}
We now choose $d=\sqrt{k-1}/10$, plug in $\Delta=1/5$ and \eqref{eq:brelate}, and make the following simplifying assumptions (which are justified by picking the constant $C$ in the theorem to be large enough):
\[
\frac{T-d\lceil\sqrt{T}\rceil}{T-1} \geq \frac{4}{5}~,~
\frac{d\lceil\sqrt{T}\rceil}{(k-1)(T-d\lceil\sqrt{T}\rceil)} \leq \frac{6}{5}\frac{d}{(k-1)\sqrt{T}} = \frac{3}{25\sqrt{(k-1)T}}
~,~
\frac{\lceil\sqrt{T}\rceil^2}{\sqrt{T}}\leq \frac{6}{5}\sqrt{T}.
\]
Performing the calculation, we get the following regret lower bound:
\[
\frac{2}{250}\sqrt{(k-1)T}+\left(\frac{1}{10}
-\frac{6}{625\sqrt{(k-1)T}}\right)\E_T[N^{>1}]
-\frac{3}{3125}\sqrt{2\left(\sqrt{(k-1)T}\right)\E_T[N^{>1}]},
\]
and lower bounding the $\sqrt{(k-1)T}$ in the middle term by $1$, we can further lower bound the expression by
\[
\frac{2}{250}\sqrt{(k-1)T}+\frac{113}{1250}\E_T[N^{>1}]
-\frac{3}{3125}\sqrt{2\left(\sqrt{(k-1)T}\right)\E_T[N^{>1}]}.
\]
How small can this expression be as a function of $\E_T[N^{>1}]$? It is easy to verify that the minimum of any function $f(x)=wx-\sqrt{vx}$ is attained for $x=v/4w^2$, with a value of $-v/4w$. Plugging in this value (for the appropriate choice of $v,w$) and simplifying, the result stated in the theorem follows.

\subsection{Proof Sketch of \thmref{thm:fixedexplorelowerbound}}

The proof idea is a reduction to the problem of distinguishing biased coins. In particular, suppose we have two Bernoulli random variables $X,Y$, one of which has a parameter $\frac{1}{2}$ and one of which has a parameter $\frac{1}{2}+\epsilon$. It is well-known that for some universal constant $c$, one cannot succeed in distinguishing the two, with a fixed probability, using only at most $c/\epsilon^2$ samples from each.

We begin by noticing that for the fixed distribution $(p_1,\ldots,p_k)$, there must be two actions each of whose probabilities is at most $1/(k-1)$ (otherwise, there are at least $k-1$ actions whose probabilities are larger than $1/(k-1)$, which is impossible). Without loss of generality, suppose these are actions $1,2$. We construct a bandit problem where the reward of action $1$ is sampled i.i.d. according to a Bernoulli distribution with parameter $\frac{1}{2}$, and the reward of action $2$ is sampled i.i.d. according to a Bernoulli distribution with parameter $\frac{1}{2}+\epsilon$, where $\epsilon = \sqrt{c'k/T}$ for some sufficiently small $c'$. The rest of the actions receive a deterministic reward of $0$. Note that this setting corresponds to the one of \thmref{thm:nonuniformupperbound}, with $\Delta=1/2$. We now run this algorithm for $T=c'k/\epsilon^2$ rounds. By picking $c'$ small enough, we can guarantee that with overwhelming probability, the algorithm samples actions $1,2$ less than $c/\epsilon^2$ times. By the information-theoretic lower bound, this implies that the algorithm must have chosen a suboptimal action for at least $\Omega(T)$ times with constant probability. Therefore, the expected regret is at least $\Omega(\epsilon T)$, which equals $\Omega(\sqrt{kT})$ by our choice of $\epsilon$.

\end{document}